\documentclass[11pt]{article}
\usepackage[utf8]{inputenc}
\usepackage[T1]{fontenc}
\usepackage[a4paper,margin=2.5cm]{geometry}
\usepackage{amsmath,amssymb,amsthm}
\usepackage{graphicx}
\usepackage{booktabs}
\usepackage{authblk}
\usepackage[hidelinks]{hyperref}
\usepackage[capitalise,noabbrev]{cleveref}
\hypersetup{bookmarksdepth=3}

\newcommand{\Radd}{R^{2}_{\mathrm{add}}}
\newcommand{\E}{\mathbb{E}}
\newcommand{\R}{\mathbb{R}}
\newcommand{\PD}{\mathrm{PD}}
\newcommand{\fhat}{\hat{f}}

\theoremstyle{plain}
\newtheorem{proposition}{Proposition}
\newtheorem{corollary}{Corollary}
\theoremstyle{remark}
\newtheorem{remark}{Remark}

\newcommand{\numALEVsCondP}{0.0817}

\newcommand{\numAlePdpDiff}{0.011}

\newcommand{\numAttainAbove}{48\%}

\newcommand{\numDiscCplxMean}{0.324}

\newcommand{\numDiscFaithMean}{0.980}

\newcommand{\numDiscInfidPerm}{0.347}
\newcommand{\numDiscNDatasets}{13}
\newcommand{\numDiscProposedMean}{0.967}

\newcommand{\numDiscSensMean}{0.958}

\newcommand{\numDiscSparsMean}{0.346}

\newcommand{\numDiscVsBestDiff}{-0.012}
\newcommand{\numDiscVsBestP}{0.371}
\newcommand{\numDiscVsBestWins}{3/13}

\newcommand{\numNDraws}{10}

\newcommand{\numNIndep}{5}

\newcommand{\numNReal}{13}

\newcommand{\numNSynthDesign}{9}
\newcommand{\numOlsALE}{0.806}
\newcommand{\numOlsCondMean}{0.781}
\newcommand{\numOlsPDP}{0.805}
\newcommand{\numOlsSHAP}{0.823}

\newcommand{\numPDPVsCondP}{0.0113}

\newcommand{\numPdpIntact}{0.702}
\newcommand{\numPdpPerm}{-0.654}

\newcommand{\numRhoCplx}{-0.090}
\newcommand{\numRhoFaith}{0.628}
\newcommand{\numRhoInfid}{0.120}
\newcommand{\numRhoSens}{-0.237}
\newcommand{\numRhoSpars}{0.190}
\newcommand{\numRhoWithinCplx}{0.003}
\newcommand{\numRhoWithinCplxHi}{0.038}
\newcommand{\numRhoWithinCplxLo}{-0.033}
\newcommand{\numRhoWithinFaith}{0.452}

\newcommand{\numRhoWithinInfid}{0.433}

\newcommand{\numSHAPVsCondP}{0.0013}

\newcommand{\numShapAleDiff}{0.085}

\newcommand{\numShapIntact}{0.798}

\newcommand{\numShapPdpDiff}{0.096}
\newcommand{\numShapPdpIndep}{-0.005}
\newcommand{\numShapPdpIndepP}{0.005}
\newcommand{\numShapPdpIndepWins}{0/5}

\newcommand{\numShapPdpReal}{0.096}
\newcommand{\numShapPdpRealP}{0.005}
\newcommand{\numShapPdpRealWins}{12/13}

\newcommand{\numShapPerm}{-0.416}

\newcommand{\numTreeShapDiffOls}{0.0003}
\newcommand{\numTreeShapDiffUnit}{0.0004}
\newcommand{\numTreeShapPOls}{0.78}
\newcommand{\numTreeShapPUnit}{0.92}

\newcommand{\numYBlackBox}{0.917}
\newcommand{\numYIntercept}{0.500}

\newcommand{\numYLinear}{0.883}

\newcommand{\numYSHAP}{0.895}
\newcommand{\numYSpline}{0.900}

\begin{document}

\title{\textbf{Evaluating Explanation Methods by the Predictors They Induce}}
\author[1]{Jacob Selb\ae{}k}
\author[1,2,3]{Hugo L. Hammer\thanks{Corresponding author: \texttt{hugo.hammer@oslomet.no}}}
\affil[1]{Department of Computer Science, Oslo Metropolitan University, Oslo, Norway}
\affil[2]{Department of Holistic Systems, SimulaMet, Oslo, Norway}
\affil[3]{Department of Plastic and Reconstructive Surgery, Oslo University Hospital, Oslo, Norway}
\date{}
\maketitle

\begin{abstract}
Explanations of machine learning models are usually judged by criteria that are hard to compare. We propose a simpler test: if an explanation really describes how a model uses its features, it should be possible to rebuild the model's predictions from it. We turn each explanation into a predictor by reading each feature's effect and adding them up, and measure how well that predictor reproduces the model on unseen data. Nothing is fitted, so the score reflects the explanation itself. The test applies to any explanation that can be written as a function of the features; we demonstrate it on partial dependence plots (PDP), accumulated local effects (ALE), SHAP and LIME. We prove that summing partial dependence curves gives the best possible additive summary of a model when its features are independent, and that this fails when they are dependent. Across \numNReal{} real datasets and \numNSynthDesign{} synthetic designs and four model families, which method scores best depends entirely on feature dependence: where features are independent SHAP is slightly worse than PDP, exactly as the theory predicts; on dependent real data SHAP leads. Some widely used quality metrics even prefer a damaged explanation to an intact one.
\end{abstract}


\section{Introduction}

Post-hoc explanation methods are now routine companions to black-box models, but the question of how to tell a good explanation from a bad one remains unsettled. The dominant evaluation criteria (faithfulness, robustness, complexity) each capture a real property, yet they assess explanations through indirect proxies: perturbation-induced changes in model output, sensitivity to input shifts, or concentration of attribution mass. Different explanation methods routinely disagree about which features drive a given prediction~\cite{krishna2022disagreement}, and several widely used fidelity metrics fail to recover the correct fidelity scores, and hence reliable rankings, even on transparent models with known ground truth~\cite{mironicolau2025comprehensive}.

This paper starts from a different observation. An explanation that claims to describe how features drive a model's predictions makes, implicitly, a falsifiable claim: it should be possible to \emph{use} that description to predict. A \emph{partial dependence curve} for feature $j$ reports, for each value $v$ that the feature might take, the average prediction the model would make if that feature were set to $v$ and the remaining features were left at the values they take in the data, so it says what effect each feature value has, marginally. A set of Shapley values says instead how much each feature contributed to one particular prediction. In both cases the content can be assembled into a function and evaluated at a point the explanation has never seen. How well that induced predictor performs on held-out data is then a direct, quantitative metric for how much predictive information the explanation carries.

This is the automated analogue of \emph{simulatability}, the criterion in which human subjects are asked to predict a model's behaviour from its explanations~\cite{lipton2018mythos,doshivelez2017towards,hase2020evaluating}. Replacing the human with a fixed, parameter-free construction removes the cost and variance of user studies at the price of measuring a narrower thing: what a specific functional form can extract, rather than what a person can understand. We think the trade is worth making, and we are explicit throughout about what it buys and what it gives up.

\paragraph{Contributions.}
\begin{enumerate}
\item \textbf{A uniform framework} (\cref{sec:framework}) converting
curve-valued explanations, namely partial dependence plots (PDP)~\cite{friedman2001greedy} and accumulated local effects (ALE)~\cite{apley2020visualizing}, and attribution-valued ones, namely SHapley Additive exPlanations (SHAP)~\cite{lundberg2017unified} and Local Interpretable Model-agnostic Explanations (LIME)~\cite{ribeiro2016why}, into predictors through a single shared construction, so that a comparison between them is not confounded by a difference in aggregation machinery.
\item \textbf{An explicit additivity reference} (\cref{sec:theory}). We make
explicit a known functional-ANOVA identification (under feature independence the unit-coefficient PDP surrogate is the $L^2$ projection of the black box onto additive functions) and draw two consequences that are not in the literature: coefficient refitting cannot help in population, so the refitting gain is an out-of-sample signal of feature \emph{dependence}; and the additive share of the model's variance reproduced by that projection, which we write $\Radd$ and define formally in Eq.~\ref{eq:radd}, bounds additive explanations only under the same independence condition, a condition we then show empirically fails on real data.
\item \textbf{A negative control for the construction itself}
(\cref{sec:results-aggregator}). Turning SHAP and LIME into curves needs a smoothing step that PDP and ALE do not, so a difference between the methods could be a difference between smoothers rather than between explanations. We therefore apply the same smoother to the black box's own predictions, with no explanation in the loop at all. It loses to every explanation method, which is what establishes that the differences we report are differences in explanation quality and not in how attributions are aggregated into global curves.
\item \textbf{A head-to-head test of the explanation-quality metrics themselves}
(\cref{sec:results-discrimination}). Every explanation-quality metric asserts an ordering over explanations, and that assertion is testable: degrade an explanation by a known amount and ask which metric notices. Applied to six metrics, the test finds that two established ones fail: infidelity is defeated by a permuted explanation, and complexity and sparseness prefer the degraded explanation to the intact one.
\item \textbf{A controlled study} (\cref{sec:experiments,sec:results}) over
\numNReal{} real datasets and \numNSynthDesign{} synthetic designs, 4 model families and 2000 evaluation cells, with null-explanation controls, a graded corruption sweep and a comparison against five established metrics.
\end{enumerate}

\subsection{Related work and positioning}\label{sec:related}

\paragraph{Explanations evaluated by predictive gain.}
The closest conceptual precedent is Pruthi et al.~\cite{pruthi2022evaluating}, who quantify explanation value by the accuracy a \emph{student} model gains when trained with the teacher's explanations available at training time but not at test time. Our framework shares the intuition that a useful explanation should confer predictive ability, and differs in a way that matters for what is being measured: our central construction is \emph{parameter-free}. A trained student is itself a flexible learner, so its gain confounds the explanation's content with the student's capacity, and a sufficiently powerful student can extract signal an explanation does not really convey. The unit-coefficient surrogate of \cref{sec:theory} has no fitted parameters at all beyond a level shift; what it achieves is attributable to the explanation. We report the refitted variants alongside precisely so this distinction is visible. Human-subject simulatability~\cite{hase2020evaluating} measures the same thing with people rather than a functional form, and is complementary: it is the right instrument for whether an explanation is \emph{understandable}, ours for whether it is \emph{informative}.

\paragraph{Interaction strength and functional decomposition.}
Molnar et al.~\cite{molnar2020complexity} define an interaction-strength metric $\mathrm{IAS}$ as the scaled approximation error of the first-order ALE decomposition against the black box, and note explicitly that with $L^2$ loss $\mathrm{IAS} = 1 - R^2$ where the targets are the model's own predictions. Our $\Radd$ is therefore \emph{the same quantity} up to the choice of ALE versus PDP main effects (which coincide under independence) and up to being computed out of sample rather than in sample. We claim no novelty for the statistic. What we add is (i) the projection theorem, which says precisely when the quantity is an attainable upper bound and when it is not, (ii) its use as the normalising reference against which explanation-derived predictors are scored, rather than as a model-complexity metric for model selection, and (iii) the empirical demonstration that under dependence it is not a bound at all: comparing each explanation's achieved fidelity against $\Radd$ on every (dataset, model, method) cell of our real data, \numAttainAbove{} of the non-PDP cells exceed it. Molnar et al.\ use the metric to choose among models, a use for which its exact status as a bound does not matter; we use it as the reference against which an explanation is judged, and for that the condition under which it bounds anything has to be stated. The issue is feature dependence, not the split on which the quantity is computed: dependence breaks the bound in sample and out of sample alike.

The underlying decomposition is Hoeffding's~\cite{hoeffding1948class}, developed for sensitivity analysis by Sobol~\cite{sobol2001global} and extended to \emph{dependent} inputs by Hooker~\cite{hooker2007generalized}, whose generalised functional ANOVA exists exactly because the marginal decomposition loses its optimality when features are correlated. That failure mode is not an inconvenience for us; it is what our \cref{cor:diagnostic} turns into a measurement.

\paragraph{Surrogate distillation.}
Fitting an interpretable global surrogate to a black box is an established practice; Distill-and-Compare~\cite{tan2018distill} is the closest instance, distilling into a GAM. The difference is directional: a distilled surrogate is \emph{fitted} to mimic the model as well as it can, whereas our surrogate is \emph{constructed} from an explanation that already exists and is not permitted to fit anything. A distillation result tells you what an additive model can do; ours tells you how much of that a given explanation method actually delivers. We include a directly fitted spline GAM as a baseline for exactly this reason.

\paragraph{Explanation-quality metrics.}
Faithfulness is typically measured by perturbing inputs according to an explanation and checking that the model responds as predicted~\cite{yeh2019infidelity,bhatt2020evaluating}; robustness by sensitivity of explanations to small input changes~\cite{alvarezmelis2018robustness}; complexity by the dispersion (entropy) of attribution mass~\cite{bhatt2020evaluating}, and sparseness by its Gini index~\cite{chalasani2020concise}. ROAR~\cite{hooker2019benchmark} evaluates attributions by retraining on feature-ablated data, again a fitted rather than constructed comparison. Toolkits such as Quantus~\cite{hedstrom2023quantus} catalogue several dozen. Much of this literature was developed for image classification with pixel attributions, and the tabular, model-agnostic setting is comparatively underexplored. We compute five of these metrics on the same explanations and splits (\cref{sec:refmetrics}) and report where our metric agrees with them and where it does not, including one case where the agreement is high enough to be worth flagging against ourselves.

A longer-standing relative is the wrapper approach to feature selection~\cite{guyon2003introduction}, which judges a feature subset by the performance of a model trained on it. Our framework shares the evaluate-by-predicting logic but does not reduce an explanation to a subset or a ranking: it consumes the full output (the shape of every curve, the value of every attribution), which is what allows PDP and ALE to be distinguished at all, since they may agree exactly on feature importance while disagreeing on the curves. General background on the methods evaluated here is in~\cite{molnar2022interpretable}; permutation importance~\cite{breiman2001random,fisher2019models} and the tree-specific SHAP estimator~\cite{lundberg2020local} are used in the implementation.

\section{Methods}\label{sec:methods}

\subsection{The framework}\label{sec:framework}

\paragraph{Notation.}
Write $X = (X_1,\dots,X_p)$ for the feature random vector, taking values in $\mathcal{X} \subseteq \R^p$ with joint law $P_X$, and $X_{-j}$ for the vector with the $j$-th coordinate removed. Lower-case $x = (x_1,\dots,x_p)$ denotes a generic point of $\mathcal{X}$ at which a function is evaluated, and $x_j$ its $j$-th coordinate. Superscripts in parentheses index the sample: $x^{(i)} = (x^{(i)}_1,\dots,x^{(i)}_p)$ is the $i$-th observation, and $x^{(i)}_j$ its $j$-th coordinate; $x^{(0)}$ is reserved for a test point at which a predictor is evaluated. We write $P_{X_j}$ for the marginal law of $X_j$, and $L^2(P)$ for the space of functions square-integrable with respect to $P$, i.e.\ those $h$ with $\int h^2 \, \mathrm{d}P < \infty$, equipped with the inner product $\langle h_1, h_2 \rangle = \E[h_1(X) h_2(X)]$. Finally $L^2_0(P_{X_j})$ denotes the centred subspace $\{h \in L^2(P_{X_j}) : \E[h(X_j)] = 0\}$.

Let $\fhat: \mathcal{X} \to \R$ be a fitted black box prediction model, where for regression $\fhat$ is the prediction and for classification it is a scalar score. An explanation method applied to $\fhat$ on a training sample produces some output $E$: a set of curves, a matrix of attributions, a collection of local models. These objects are not comparable to one another, and none of them is a number, so none can be scored directly.

In this paper we suggest a framework making us able to evaluate and compare them. Each output $E$ is mapped to a function $g_E : \mathcal{X} \to \R$, a predictor built out of nothing but the explanation, which can be evaluated at points the explanation never saw. The quality of $E$ is then \emph{defined} as the out-of-sample accuracy of $g_E$:
\begin{equation}\label{eq:framework}
  Q(E) \;=\; R^2\bigl(g_E,\, t\bigr) \quad\text{on a held-out split,}
\end{equation}
where the target $t$ is taken in turn to be
\begin{itemize}
\item the original outcome $y$, how much of the \emph{task} the explanation
retains;
\item the black box's own score $\fhat(x)$, a global, out-of-sample
\emph{fidelity} metric.
\end{itemize}
So $g_E$ is not itself the object of interest and is never proposed as a model anyone should deploy. It is the instrument through which $E$ is measured: a faithful explanation carries enough of the black box to rebuild its predictions, and an uninformative one does not. Eq.~\ref{eq:framework} is the evaluation metric this paper proposes, and everything that follows is an examination of it.

Crucially, $E$ is computed on the training split only; the test split is never seen by the explanation, so nothing about the evaluation is circular.

\subsubsection{Curve-valued explanations}

PDP~\cite{friedman2001greedy} and ALE~\cite{apley2020visualizing} return one one-dimensional curve per feature or higher-dimensional curves for combinations of features. The two-dimensional version is a standard tool for inspecting interactions between pairs of features. We restrict attention to the one-dimensional case throughout, because it is by far the most common in practice and because it is the case in which the additivity reference of \cref{sec:theory} has an exact characterisation. Nothing in the framework prevents higher-order curves from being used, and \cref{sec:limitations} returns to what including them would change. Writing $c_j$ for the curve of feature $j$, centred so that $\E[c_j(X_j)] = 0$ over the training data, the induced predictor $g_E$ (the first concrete instance of the general construction above) is the additive surrogate
\begin{equation}\label{eq:surrogate}
  g_{\alpha}(x) \;=\; \alpha_0 \;+\; \sum_{j=1}^{p}\alpha_j\, c_j(x_j),
  \qquad x \in \mathcal{X},\; \alpha \in \R^{p+1}.
\end{equation}
We consider three variants, differing only in how $\alpha$ is set:
\begin{align}
  \text{\emph{unit}:}\quad & \alpha_j = 1 \;(j \ge 1), \qquad
    \alpha_0 = \bar{t} - \tfrac{1}{n}\textstyle\sum_{i}\sum_j c_j(x^{(i)}_j),
    \label{eq:unit}\\
  \text{\emph{ols}:}\quad & \hat\alpha = \arg\min_{\alpha \in \R^{p+1}}
    \textstyle\sum_{i=1}^{n}\bigl(t^{(i)} - g_\alpha(x^{(i)})\bigr)^2,
    \label{eq:ols}\\
  \text{\emph{ridge}:}\quad & \hat\alpha_\lambda = \arg\min_{\alpha}
    \textstyle\sum_{i=1}^{n}\bigl(t^{(i)} - g_\alpha(x^{(i)})\bigr)^2
    + \lambda\lVert \alpha_{1:p} \rVert_2^2 ,
    \label{eq:ridge}
\end{align}
where $t^{(i)}$ is the target on the training split (either $y^{(i)}$ or $\fhat(x^{(i)})$) and $\bar t$ its mean. Only Eq.~\ref{eq:unit} fits nothing beyond a level: the curves enter with unit weight, so whatever the surrogate achieves is attributable to the explanation and not to a fit. \Cref{sec:theory} shows that this is not merely a convenient restriction but, under an independence condition, the optimal choice.

\paragraph{Why a linear aggregator, and why an unfitted one.}
The obvious alternative is to feed the curve values $c_1(x_1),\dots,c_p(x_p)$ into a flexible learner (a gradient-boosted ensemble, say) and report its accuracy. We deliberately do not, for two reasons.

The first is attribution of credit. A flexible aggregator can recover information that the explanation does not contain: given the transformed features it can re-learn interactions, rescale non-monotonically, and compensate for a badly estimated curve. Its accuracy would then be a joint property of the explanation and of the aggregator's own capacity, and a comparison between two explanation methods would be confounded by how much work the aggregator did for each.

This is precisely where our construction departs from Pruthi et al.~\cite{pruthi2022evaluating}, the closest precedent. They measure an explanation by the accuracy a trainable \emph{student} model gains when the teacher's explanations are available during training. The student is exactly the kind of flexible aggregator we are describing, and its capacity is a free variable: a stronger student extracts more from the same explanation, and can also extract structure from the training data that the explanation never supplied. Their metric therefore confounds the quality of the explanation with the capacity of the student, and two explanations can swap places under a change of student architecture. Fixing every coefficient at unity removes that degree of freedom entirely: there is no student to tune, and no capacity for a comparison to be sensitive to. Eq.~\ref{eq:unit} has no free parameters at all beyond a single intercept: the curves are added up exactly as an analyst reading the plots would add them up. Whatever $R^2$ it attains is therefore a statement about the explanation.

The second is interpretability of the number. Because Eq.~\ref{eq:unit} is parameter-free, its deficit $1 - R^2$ has a meaning we can characterise exactly: \cref{sec:theory} identifies it, under feature independence, with the share of the model's variance that no additive description could capture. No such characterisation is available for the output of an arbitrary learner.

The fitted variants Eqs.~\ref{eq:ols} and~\ref{eq:ridge} are included not as competitors but as a diagnostic: they are the smallest possible relaxation of the unit constraint, and \cref{sec:theory} shows that what they gain over it is itself an interpretable quantity.

\subsubsection{Attribution-valued explanations}\label{sec:attribution}

SHAP~\cite{lundberg2017unified} and LIME~\cite{ribeiro2016why} do not return functions. They return, for each explained observation, a vector of $p$ numbers describing that observation only, so there is nothing to evaluate at a point the explanation has not seen. Converting them into a predictor therefore requires one additional modelling step, and since that step is not part of either method we apply exactly the same one to both.

\paragraph{From attributions to curves.}
Let $\phi_j(x^{(i)})$ denote the attribution that the explanation method assigns to feature $j$ at observation $x^{(i)}$, and suppose attributions have been computed on $m$ training observations. For each feature $j$ we treat the $m$ pairs
\begin{equation}\label{eq:dependence}
  \mathcal{S}_j \;=\; \bigl\{\, \bigl(x_j^{(i)},\ \phi_j(x^{(i)})\bigr)
  \,\bigr\}_{i=1}^{m}
\end{equation}
as a scatter of attribution against feature value, and estimate from it a univariate function $\hat c_j : \R \to \R$. The estimator is a quantile-binned average. Writing $q_j(u)$ for the empirical $u$-quantile of feature $j$ on the training split, we partition its range into $B$ bins with edges $q_j(0) = e_{j,0} < e_{j,1} < \dots < e_{j,B} = q_j(1)$, $e_{j,b} = q_j(b/B)$, so that each bin holds approximately $m/B$ observations. Let $I_{j,b} = \{ i : e_{j,b-1} < x^{(i)}_j \le e_{j,b} \}$ be the index set of bin $b$. The curve is then defined at the within-bin mean feature value by the within-bin mean attribution,
\begin{equation}\label{eq:binned}
  \tilde{x}_{j,b} \;=\; \frac{1}{|I_{j,b}|}\sum_{i \in I_{j,b}} x^{(i)}_j,
  \qquad
  \hat c_j(\tilde{x}_{j,b}) \;=\; \frac{1}{|I_{j,b}|}
  \sum_{i \in I_{j,b}} \phi_j(x^{(i)}),
\end{equation}
and extended to all of $\R$ by linear interpolation between consecutive knots $(\tilde{x}_{j,b}, \hat c_j(\tilde{x}_{j,b}))$, held constant beyond the outermost knots. Bins containing fewer than a minimum count are merged with their neighbour, and a feature taking only two values is given one knot per value. Finally $\hat c_j$ is centred so that $\frac{1}{m}\sum_i \hat c_j(x^{(i)}_j) = 0$, matching the convention for PDP and ALE curves, and the resulting $\hat c_j$ are substituted into \cref{eq:surrogate}.

Two remarks on what \cref{eq:binned} is. First, it is a conditional-mean estimator: as $B$ grows with $m$, $\hat c_j(v)$ approaches $\E[\phi_j(X) \mid X_j = v]$. This matters for interpretation (a conditional average can track structure that the marginal partial dependence curve cannot when features are dependent), and it is why \cref{sec:results-aggregator} reports a control in which the same estimator is applied to the model's own output rather than to attributions. Second, it introduces the only tuning parameter in the attribution arm, the bin count $B$; we use $B = 20$ with a minimum of ten observations per bin throughout and note the absence of a sensitivity sweep as a limitation.

\paragraph{SHAP.}
The attribution is the Shapley value itself, $\phi_j(x^{(i)})$, and $\mathcal{S}_j$ is precisely the SHAP dependence plot of Lundberg et al.~\cite{lundberg2020local}. Summing per-feature Shapley curves into an additive model is not new (Bordt and von Luxburg~\cite{bordt2023shapley} establish a correspondence between Shapley values and generalized additive models), though our use is narrower, since we need only a function that can be evaluated out of sample.

\paragraph{LIME.}
LIME returns not attributions but a local surrogate: around each explained observation $x^{(i)}$ it fits an interpretable model $L_i$ to perturbations of that observation, weighted by proximity. With the usual linear choice, $L_i(x) = \beta_{i,0} + \sum_{j=1}^{p} \beta_{i,j} x_j$, the coefficient $\beta_{i,j}$ is a slope and not a contribution, so it is not yet on the scale that \cref{eq:dependence} expects. We convert it in the standard way, by evaluating the contribution of feature $j$ at $x^{(i)}$ relative to a reference point (we use the training mean $\bar{x}_j$), giving
\begin{equation}\label{eq:limeattr}
  \phi_j(x^{(i)}) \;=\; \beta_{i,j}\bigl(x^{(i)}_j - \bar{x}_j\bigr).
\end{equation}
This is the amount by which the local model at $x^{(i)}$ says feature $j$ moves the prediction away from its value at the reference, and it is directly comparable to a Shapley value, which is likewise a deviation from a baseline. Substituting \cref{eq:limeattr} into \cref{eq:dependence} and then into \cref{eq:binned} yields $\hat c_j$ exactly as for SHAP; nothing else in the pipeline differs between the two methods.

Because LIME's local models are themselves functions, LIME also admits a construction that does not pass through \cref{eq:binned} at all: predict at a new point by an inverse-distance weighted average of the local models evaluated there,
\begin{equation}\label{eq:limeidw}
  \hat y(x^{(0)}) \;=\; \frac{\sum_{i=1}^{m} \delta(x^{(i)},x^{(0)})^{-1} L_i(x^{(0)})}
                         {\sum_{i=1}^{m} \delta(x^{(i)},x^{(0)})^{-1}},
  \qquad \delta(x^{(i)},x^{(0)}) = \lVert x^{(i)} - x^{(0)} \rVert_2^2 .
\end{equation}
We report this as local-IDW alongside the additive construction. It is legitimate here precisely because $L_i$ is a function; the superficially similar construction on SHAP values is not, for the reason given in \cref{sec:degenerate}.

\Cref{eq:limeattr} is specific to a linear local model, and it is worth being explicit about what happens otherwise. The local-IDW construction Eq.~\ref{eq:limeidw} needs only that $L_i$ be a function of $x$, so it applies unchanged if the local surrogate is a shallow decision tree, a rule list, or any other evaluable model. The additive construction is more demanding: it needs a per-feature contribution $\phi_j(x^{(i)})$, which a linear model supplies through its coefficients but a tree does not, since a tree's output is not separable across features. For a non-additive local surrogate one would have to extract per-feature contributions by a further decomposition (for instance the Shapley values of $L_i$ itself, which is cheap because $L_i$ is small), and the choice of that decomposition would then become part of the construction being evaluated. We therefore restrict attention to the linear local model, which is both the LIME default and the case in which the conversion is unambiguous.

\subsubsection{A construction to avoid}\label{sec:degenerate}

It is tempting to apply inverse-distance weighting to SHAP values in the same way we applied it to LIME's local models in Eq.~\ref{eq:limeidw}. The motivation is identical: a SHAP explanation, like a LIME explanation, is attached to particular training points, so to predict at a new point one interpolates between the nearby ones. Concretely, one takes an inverse-distance weighted average of the training attribution sums $S_i = \sum_j \phi_j(x^{(i)})$ and adds back the baseline,
\begin{equation}\label{eq:degenerate}
  \hat{y}(x^{(0)}) \;=\; \E[\fhat] + \sum_i w_i S_i,
  \qquad \textstyle\sum_i w_i = 1 .
\end{equation}
The two constructions are not, however, on the same footing, and this one is degenerate. SHAP satisfies local accuracy exactly, $S_i = \fhat(x^{(i)}) - \E[\fhat]$, so \cref{eq:degenerate} collapses to
\begin{equation}
  \hat{y}(x^{(0)}) \;=\; \sum_i w_i\, \fhat(x^{(i)}),
\end{equation}
which is inverse-distance $k$NN regression on the black box's own predictions. The individual attributions cancel: the predictor is invariant to how the attribution mass is distributed across features, and would return identical values for a completely fabricated attribution matrix with the same row sums. Worse, it fits its anchor points by tautology (at zero distance the weight diverges), so it exhibits a large train--test gap that is an artefact of the construction and says nothing about SHAP.

The contrast with LIME is what makes the point sharp. Eq.~\ref{eq:limeidw} interpolates between local \emph{models}, each of which is a function of $x$ carrying per-feature slopes, so the interpolation preserves the content of the explanation. Eq.~\ref{eq:degenerate} interpolates between attribution \emph{sums}, and local accuracy has already fixed every one of those sums to a number that does not depend on the attributions at all. The construction therefore cannot see the explanation it is supposedly evaluating. We flag this because it is the natural first thing to reach for, and because a reader encountering a large SHAP train--test gap in the literature should ask which construction produced it.

\subsection{Theory: the additive projection}\label{sec:theory}

Recall from \cref{sec:framework} that $L^2(P_X)$ is the space of functions square-integrable against the joint law of the features, that $P_{X_j}$ is the marginal law of $X_j$, and that $L^2_0(P_{X_j})$ is the centred subspace of $L^2(P_{X_j})$. Define the space of square-integrable \emph{additive} functions
\begin{equation}\label{eq:addspace}
  \mathcal{A} \;=\; \Bigl\{\, a_0 + \sum_{j=1}^p a_j(X_j) \;:\;
  a_0 \in \R,\ a_j \in L^2_0(P_{X_j}) \,\Bigr\} \;\subseteq\; L^2(P_X),
\end{equation}
that is, all functions expressible as a constant plus a sum of one-dimensional functions of the individual features, each with mean zero. $\mathcal{A}$ contains every surrogate $g_\alpha$ of \cref{eq:surrogate}, for any $\alpha$ and any centred curves, and is far larger than that finite-dimensional family: it places no restriction on the shape of each $a_j$ beyond square-integrability.

Let
\begin{equation}\label{eq:pd}
  \PD_j(x_j) \;=\; \E\bigl[\fhat(x_j, X_{-j})\bigr]
  \;=\; \int \fhat(x_j, x_{-j})\, \mathrm{d}P_{X_{-j}}(x_{-j})
\end{equation}
denote the partial dependence function of feature $j$ (the average prediction when feature $j$ is held at $x_j$ and the others are drawn from their \emph{marginal} joint distribution), with centred version $c_j(x_j) = \PD_j(x_j) - \E[\fhat(X)]$.

\begin{proposition}[Additive projection]\label{prop:projection}
Suppose $X_1,\dots,X_p$ are mutually independent and $\fhat \in L^2(P_X)$. Then the orthogonal projection of $\fhat$ onto $\mathcal{A}$, i.e.\ the minimiser of $\E[(\fhat(X) - a(X))^2]$ over $a \in \mathcal{A}$, is
\begin{equation}\label{eq:fadd}
  \fhat_{\mathrm{add}}(x) \;=\; \E[\fhat(X)] \;+\; \sum_{j=1}^{p} c_j(x_j),
\end{equation}
which is exactly the unit-coefficient surrogate $g_{\mathbf{1}}$ of \cref{eq:unit}. Consequently
\begin{equation}\label{eq:radd}
  \Radd \;:=\; 1 - \frac{\E\bigl[(\fhat(X) -
  \fhat_{\mathrm{add}}(X))^2\bigr]}{\operatorname{Var}(\fhat(X))}
\end{equation}
is the fraction of the model's variance that is additive. It depends only on $\fhat$ and $P_X$, not on any explanation method.
\end{proposition}

\begin{proof}
Under independence, marginalising over $X_{-j}$ coincides with conditioning: $\PD_j(x_j) = \E[\fhat(x_j,X_{-j})] = \E[\fhat(X)\mid X_j = x_j]$, since $X_{-j}$ has the same law conditionally on $X_j = x_j$ as unconditionally. Under a product measure $\fhat$ admits the Hoeffding decomposition~\cite{hoeffding1948class}
\begin{equation}
  \fhat(X) \;=\; \sum_{S \subseteq \{1,\dots,p\}} \fhat_S(X_S),
  \qquad \E\bigl[\fhat_S(X_S) \mid X_T\bigr] = 0
  \ \text{ whenever } S \not\subseteq T,
\end{equation}
whose components are mutually orthogonal in $L^2(P_X)$ and whose first-order terms are $\fhat_{\{j\}}(x_j) = \E[\fhat \mid X_j = x_j] - \E[\fhat] = c_j(x_j)$. The space defined in Eq.~\ref{eq:addspace} is the direct sum of the order-$\le 1$ components,
\begin{equation}\label{eq:directsum}
  \mathcal{A} \;=\; \R \,\oplus\, L^2_0(P_{X_1}) \,\oplus\, \cdots
  \,\oplus\, L^2_0(P_{X_p})
  \;=\; \R \oplus \bigoplus_{j=1}^{p} L^2_0(P_{X_j}),
\end{equation}
where $\oplus$ denotes a direct sum of mutually orthogonal subspaces: every $a \in \mathcal{A}$ decomposes uniquely into a constant plus one centred function of each feature, and functions drawn from different summands are orthogonal in $L^2(P_X)$. The space is closed as a finite direct sum of closed subspaces, so the projection exists and is unique. Every higher-order component is orthogonal to it: for $|S| \ge 2$ and any $j$ we have $\E[\fhat_S \mid X_j] = 0$, hence $\langle \fhat_S, a_j \rangle = \E[a_j(X_j)\E[\fhat_S \mid X_j]] = 0$ for all $a_j \in L^2_0(P_{X_j})$, and $\langle \fhat_S, 1\rangle = 0$. The projection therefore retains exactly the constant and first-order terms, $\Pi_{\mathcal{A}}\fhat = \E[\fhat] + \sum_j c_j = g_{\mathbf{1}}$.
\end{proof}

\paragraph{What $\Radd$ means.}
It is worth drawing out the interpretation, because it is what makes the quantity useful rather than merely definable. \Cref{prop:projection} says that the unit-coefficient partial-dependence surrogate (the object one obtains by reading each PDP curve at the observed feature value and adding the readings up) is not an approximation to the best additive summary of the model. It \emph{is} the best additive summary, in the mean-square sense, under independence. The variance it fails to explain, $1 - \Radd$, is therefore not a deficiency of partial dependence as an explanation device: it is the variance that no additive description of $\fhat$ could have captured, because it lives entirely in the interaction terms $\fhat_S$, $|S| \ge 2$.

This gives a concrete reading of what is lost when a model is explained through marginal curves. Presenting a practitioner with $p$ one-dimensional plots rather than the model itself costs exactly $1 - \Radd$ of the model's variance, the same loss that would have been incurred by fitting an additive model to the data in the first place and dispensing with the black box. In that sense the choice between "explain a complex model additively'' and "fit an additive model'' is, on this criterion, not a choice at all: they forgo the same information. What differs is that the black box retains the interaction structure internally and can still use it to predict, while the additive model cannot. $\Radd$ is the price of the summary, and it is a property of the model being summarised.

\begin{corollary}[Refitting cannot help]\label{cor:refit}
Under the conditions of \cref{prop:projection}, for every $\alpha \in \R^{p+1}$,
\begin{equation}
  \E\bigl[(\fhat - g_{\mathbf{1}})^2\bigr] \;\le\;
  \E\bigl[(\fhat - g_{\alpha})^2\bigr].
\end{equation}
The population-optimal coefficients are $\alpha_j \equiv 1$, uniquely so when the $c_j$ are linearly independent and non-degenerate; a feature with no main effect ($c_j \equiv 0$) leaves its $\alpha_j$ unidentified.
\end{corollary}

\begin{proof}
Each $g_\alpha$ of \cref{eq:surrogate} is a constant plus a sum of univariate centred functions, hence $g_\alpha \in \mathcal{A}$. By \cref{prop:projection}, $g_{\mathbf 1} = \Pi_{\mathcal{A}}\fhat$, and the claim is the defining minimality of an orthogonal projection over its subspace.
\end{proof}

\Cref{cor:refit} is worth pausing on. It says that fitting coefficients to the curves (the \emph{ols} variant of Eq.~\ref{eq:ols}) cannot improve on leaving them at unity, provided features are independent. Any improvement that ordinary least squares does deliver must therefore come from one of two places: it is exploiting sample noise, or it is compensating for a violation of the independence assumption. The second is the useful one, and it can be measured.

\begin{corollary}[Refitting gain as a dependence diagnostic]\label{cor:diagnostic}
Let $\hat\alpha$ be the ordinary-least-squares coefficients of Eq.~\ref{eq:ols} fitted on the training split against the target $\fhat$, and define the out-of-sample refitting gain on a held-out split as
\begin{equation}\label{eq:gain}
  \Delta_{\mathrm{OLS}} \;=\; R^2\bigl(g_{\hat\alpha}\bigr)
  \;-\; R^2\bigl(g_{\mathbf{1}}\bigr),
  \qquad
  \Delta^{\mathrm{rel}}_{\mathrm{OLS}} \;=\;
  \frac{R^2\bigl(g_{\hat\alpha}\bigr) - R^2\bigl(g_{\mathbf{1}}\bigr)}
       {R^2\bigl(g_{\mathbf{1}}\bigr)},
\end{equation}
the second being the same quantity expressed as a proportion of what the parameter-free surrogate already achieves, which is the form we report because it is comparable across settings whose overall fidelity differs. Under the conditions of \cref{prop:projection}, $\Delta_{\mathrm{OLS}} \to 0$ as the training and evaluation samples grow. Under dependence $\PD_j \neq \E[\fhat \mid X_j]$ in general, so $g_{\mathbf 1}$ need not be the projection, the population-optimal $\alpha$ within $\operatorname{span}\{1, c_1,\dots,c_p\}$ need not be $\mathbf{1}$, and $\Delta_{\mathrm{OLS}} \ge 0$ in population with strict inequality whenever it differs.
\end{corollary}

What makes \cref{cor:diagnostic} useful in practice is what it does \emph{not} require. Feature dependence is a property of the joint law $P_X$, and establishing it directly means estimating that law, or at least testing for dependence among $p$ variables, a problem that is hard in more than a few dimensions, that pairwise correlations answer only partially since they miss non-linear dependence, and that a practitioner working with a fitted model and a data sample is rarely in a position to solve. $\Delta_{\mathrm{OLS}}$ sidesteps it. Both terms in Eq.~\ref{eq:gain} are out-of-sample $R^2$ values of predictors the analyst has already built, so the diagnostic costs one extra least-squares fit on curves that were computed anyway, and it never touches $P_X$. In other words, the very condition under which the parameter-free construction is licensed can be checked with the construction itself: fit the coefficients, see whether the gain is bigger than the model family's calibrated null, and if it is, treat $\Radd$ as a reconstruction rather than a bound.

Two caveats belong with the corollary rather than in a later discussion. First, $\Delta_{\mathrm{OLS}}$ is confounded with anything else that mis-scales a curve: finite-sample noise in the partial dependence estimate, grid coarseness, and most conspicuously a regularised explanation method whose coefficients are attenuated. It is a dependence signal only when the curve estimator is otherwise well calibrated. Second, the finite-sample out-of-sample $\Delta_{\mathrm{OLS}}$ carries a negative bias, so its null is not exactly zero and is model-dependent; it requires per-family calibration rather than assumption. \Cref{sec:results-theory} tests the corollary and reports both.

\begin{remark}[Relation to $\mathrm{IAS}$]\label{rem:ias}
$\Radd = 1 - \mathrm{IAS}$ in the sense of Molnar et al.~\cite{molnar2020complexity}, modulo ALE versus PDP main effects and in-sample versus out-of-sample evaluation. \Cref{prop:projection} supplies the condition (independence) under which the quantity is the optimal additive approximation and hence a genuine ceiling. When that condition fails, $\Radd$ is neither an upper bound nor bounded below by zero, and we observe negative values on dependent data. This is precisely the regime Hooker~\cite{hooker2007generalized} introduced the generalised functional ANOVA to handle, and it is a caveat that the in-sample, model-selection use of $\mathrm{IAS}$ does not have to confront.
\end{remark}

\subsection{Experimental design}\label{sec:experiments}

\paragraph{Data}
The study uses \numNReal{} tabular benchmarks from OpenML~\cite{vanschoren2014openml}, selected for a high share of continuous features, and \numNSynthDesign{} synthetic designs whose structure is known by construction. \Cref{tab:datasets} lists the real datasets; the synthetic designs are described below and named where they are used.

The synthetic suite is generated from a single family with two free knobs, so that feature dependence and interaction strength can be varied independently. Features are drawn from a mean-zero, unit-variance multivariate Gaussian distribution with equicorrelated components,
\begin{equation}\label{eq:dgpX}
  X \sim \mathcal{N}_p(\mathbf{0}, \Sigma_\rho), \qquad
  (\Sigma_\rho)_{jk} = \rho^{\,\textbf{I}(j \neq k)},
\end{equation}
where $\textbf{I}(\text{Boolean})$ is the indicator function returning 1 if Boolean is true and 0 else. Using $\rho = 0$ gives mutually independent features (the condition of \cref{prop:projection}) and $\rho > 0$ violates it. The signal is built from an additive part and a pure-interaction part,
\begin{align}
  A(x) &= 1.2\,x_1 + 0.9\sin(1.5 x_2) + 0.8\,(x_3^2 - 1)
          + \tanh(2 x_4) + \sum_{j \ge 5} \tfrac{0.6}{\sqrt{j}}\, x_j,
          \label{eq:dgpA}\\
  H(x) &= x_1 x_2 + x_3 x_4 + 0.8\, x_5 x_6,
          \label{eq:dgpH}
\end{align}
each standardised to zero mean and unit variance over the sample.

The point of splitting the signal this way is that at $\rho = 0$ the two parts are orthogonal, so the interaction weight below controls the additive share exactly rather than approximately. This is worth spelling out. Take a single product term $X_kX_l$ with $k \neq l$, and suppose the features are independent with mean zero. Conditioning on any one feature $X_j$ leaves at least one of the two factors free, and that free factor is independent of everything conditioned on, so
\begin{equation}\label{eq:hzero}
  \E\bigl[X_kX_l \mid X_j\bigr] =
  \begin{cases}
    X_k\,\E[X_l] = 0, & j = k,\\
    \E[X_k]\,X_l = 0, & j = l,\\
    \E[X_k]\,\E[X_l] = 0, & j \notin \{k,l\},
  \end{cases}
\end{equation}
and hence $\E[H(X) \mid X_j] = 0$ for every $j$, and $\E[H(X)] = 0$. Two consequences follow. First, $H$ is orthogonal to the additive space $\mathcal{A}$ of Eq.~\ref{eq:addspace}: for any $a = a_0 + \sum_j a_j(X_j)$ in $\mathcal{A}$, the tower property gives
\begin{equation}
  \langle H, a \rangle
  = a_0\,\E[H(X)] + \sum_{j=1}^{p} \E\bigl[a_j(X_j)\,\E[H(X) \mid X_j]\bigr]
  = 0 .
\end{equation}
Second, $H$ contributes nothing to any partial dependence function, since under independence $\PD_j$ is the conditional mean and $\E[H(X) \mid X_j = x_j] = 0$. So the whole of $H$ lives in the part of the model that no additive explanation can reach, which is precisely the property the interaction sweep needs in order to be a controlled experiment. Note that Eq.~\ref{eq:hzero} uses independence twice; at $\rho > 0$ neither conclusion survives, which is why only the $\rho = 0$ designs validate the estimator.

The two parts are combined with an interaction weight $\tau \in [0,1]$,
\begin{equation}\label{eq:dgpf}
  f_{\tau,\rho}(x) \;=\; \kappa\Bigl(\sqrt{1-\tau}\,A(x)
  \;+\; \sqrt{\tau}\,H(x)\Bigr),
\end{equation}
with $\kappa$ fixed so that $\operatorname{sd}(f_{\tau,\rho}) = 2.5$. Because $A$ and $H$ are orthogonal at $\rho = 0$, the additive share of the signal variance is then $1 - \tau$ by construction. Targets are $y \mid x \sim \mathrm{Bernoulli}\bigl(\mathrm{logit}^{-1} f_{\tau,\rho}(x)\bigr)$ for classification and $y = f_{\tau,\rho}(x) + \varepsilon$ with $\varepsilon \sim \mathcal{N}(0, (0.3\,\operatorname{sd}(f))^2)$ for regression.

The suite then consists of: an \emph{interaction sweep} $\tau \in \{0, 0.10, 0.25, 0.50, 0.75\}$ at $\rho = 0$, giving designed additive shares from $1.00$ down to $0.255$ with independent features; a \emph{correlation sweep} $\rho \in \{0.30, 0.60, 0.85\}$ at fixed $\tau = 0.20$; and a regression variant. All use $n = 2000$ and $p = 8$.

Feature encoding matters more than it appears. On a two-valued column a curve is an \emph{affine} function of that column, so the explanation transform is a no-op and a linear model on the raw dummy spans the same space. A design dominated by binary columns therefore \emph{cannot} discriminate between curve-based explanation methods, whatever their relative merits. We accordingly exclude any dataset whose encoded design is more than 60\% binary, which removes \texttt{credit\_g} from the real data and the binary-dominated synthetic design from the suite; including them would have diluted every comparison with cells that are uninformative by construction. For the datasets that remain we report the binary fraction, drop reference levels when one-hot encoding so that the design matrix is full rank, and never discretise continuous features.

\begin{table}[htbp]
\centering
\small
\caption{The real datasets used in the study. $n$ is the number of rows, $p_{\mathrm{raw}}$ the number of variables before encoding, $p$ the width of the one-hot encoded design matrix, and \emph{Bin.\ frac.} the share of encoded columns taking at most two values. \emph{Task} is clf for classification and reg for regression, and \emph{Pos.\ rate} the marginal rate of the positive class, shown as a dash for regression.}
\label{tab:datasets}
\begin{tabular}{lrrrrrr}
\toprule
Dataset & $n$ & $p_{\mathrm{raw}}$ & $p$ & Bin.\ frac. & Task & Pos.\ rate \\
\midrule
abalone & 4177 & 8 & 9 & 0.22 & reg & -- \\
boston\_housing & 506 & 13 & 20 & 0.45 & reg & -- \\
breast\_cancer\_wisc & 699 & 9 & 9 & 0.00 & clf & 0.345 \\
concrete & 1030 & 8 & 8 & 0.00 & reg & -- \\
cpu\_activity & 8192 & 21 & 21 & 0.00 & reg & -- \\
diabetes\_pima & 768 & 8 & 8 & 0.00 & clf & 0.349 \\
heart\_statlog & 270 & 13 & 13 & 0.23 & clf & 0.444 \\
phoneme & 5404 & 5 & 5 & 0.00 & clf & 0.293 \\
qsar\_biodeg & 1055 & 41 & 41 & 0.07 & clf & 0.337 \\
spambase & 4601 & 57 & 57 & 0.00 & clf & 0.394 \\
wind\_speed & 6574 & 14 & 14 & 0.00 & reg & -- \\
wine\_quality\_red & 1599 & 11 & 11 & 0.00 & clf & 0.136 \\
wine\_quality\_white & 4898 & 11 & 11 & 0.00 & clf & 0.216 \\
\bottomrule
\end{tabular}
\end{table}

\paragraph{Models and protocol.}
Four black-box families, namely Random Forest, Gradient Boosting, Multilayer Perceptron (MLP), Radial Basis Function Support Vector Machine (SVM RBF), each tuned by randomised search with 5-fold cross-validation within every repeat, and 20 repeated stratified train/test splits (75\%/25\%). This gives $25\times4\times20 = 2000$ evaluation cells and 6.31 million recorded metric values, with no failed fits. Classification is explained on the probability scale, and motivation in \cref{sec:limitations}.

\paragraph{Controls.}
A new metric must be shown to measure something. We include: \emph{null explanations}, a graded \emph{corruption sweep}, and \emph{baselines}. Stated precisely, let a curve $c_j$ be represented by its values $v_{j,1},\dots,v_{j,K_j}$ at knots $u_{j,1} < \dots < u_{j,K_j}$, and write $s_j = \operatorname{sd}(v_{j,1},\dots,v_{j,K_j})$ for its amplitude. The three curve manipulations are
\begin{align}
  \text{permuted:}\quad & v^{\mathrm{perm}}_{j,k} = v_{j,\pi_j(k)},
    \quad \pi_j \text{ a uniform random permutation of } \{1,\dots,K_j\},
    \label{eq:nullperm}\\
  \text{randomised:}\quad & v^{\mathrm{rand}}_{j,k} = \eta_{j,k}
    - \bar\eta_j, \quad \eta_{j,k} \sim \mathcal{N}(0, s_j^2),
    \label{eq:nullrand}\\
  \text{corrupted at } \sigma:\quad & v^{\sigma}_{j,k} = v_{j,k}
    + \zeta_{j,k}, \quad \zeta_{j,k} \sim \mathcal{N}(0, \sigma^2 s_j^2).
    \label{eq:corrupt}
\end{align}
\Cref{eq:nullperm} is the sharp null: it leaves the multiset of effect magnitudes exactly unchanged and destroys only the mapping from feature value to effect, so a metric that does not collapse under it is not reading the curve's shape. \Cref{eq:nullrand} additionally discards the magnitudes. \Cref{eq:corrupt} interpolates between the intact curve and noise, with $\sigma \in \{0, 0.1, 0.25, 0.5, 1, 2, 4\}$ in units of each curve's own amplitude, so that features with large and small effects are corrupted comparably. Baselines span an intercept-only floor ($g \equiv \bar t$), ordinary least squares on the raw encoded features, a spline additive model fitted directly to the target, and the black box itself.

\paragraph{Statistics.}
The unit of replication is the \emph{dataset}. Repeats within a dataset share most of their training rows and the four model families share the data entirely, so treating individual cells as independent would be pseudoreplication; every reported test therefore averages to one value per dataset and is paired over those. For a replicated synthetic design the unit is the \emph{design}, its \numNDraws{} independent draws being averaged into it first, so that one generative family does not cast \numNDraws{} times the votes of a real dataset. Holm correction~\cite{holm1979simple} is applied over the six pairwise comparisons, which are the confirmatory family; everything else (the controls, the estimator comparison, the correlations against established metrics) is
exploratory and reported uncorrected. Intervals in the pooled tables are 95\%
$t$ intervals over repeats and describe split-to-split variation only; the between-dataset variation that the paired tests use is substantially larger.

\paragraph{Smoothing parameters.}
PDP uses a 25-point quantile grid over a 300-row background; ALE uses 20 quantile intervals; the SHAP and LIME dependence curves use 20 quantile bins with a 10-row minimum. LIME is run at library defaults with \texttt{discretize\_continuous=False}, 1000 perturbations and 300 anchor points; TreeSHAP is run \texttt{interventional} (see \cref{sec:results-aggregator}). We have not swept these; a sensitivity analysis over the bin count is the most obvious omission.

\section{Results}\label{sec:results}

Six experiments follow, in an order that builds on itself, and they fall into three groups.

The first three ask whether the proposed metric can be trusted at all, and must come before any comparison rests on it. \Cref{sec:results-validity} checks that the metric responds to the content of an explanation, since a metric that survives having its input destroyed cannot rank anything. \Cref{sec:results-theory} checks the prediction the theory makes about refitting. \Cref{sec:results-aggregator} checks that a difference between two explanation methods is a difference between the explanations themselves, rather than between the machinery used to turn them into predictors.

\Cref{sec:results-main} then uses the metric for the comparison the paper set out to make, between PDP, ALE, SHAP and LIME.

The last two turn outward, to the other metrics in the literature. \Cref{sec:refmetrics} asks how ours relates to five established ones, and \cref{sec:results-discrimination} asks whether it detects a degraded explanation any better than they do.

Unless a subsection says otherwise, results are computed on the \numNReal{} real datasets only. The one exception is \cref{sec:results-theory}, which needs a known feature correlation and therefore uses the synthetic designs; it says so explicitly.

\subsection{The metric responds to explanation content}\label{sec:results-validity}

It is important to establish, before any comparison between explanation methods, that the metric is reading the explanation at all. A predictor built from an explanation could score well for reasons that have nothing to do with the explanation's content: the feature values alone carry information, and a construction that leaked any of it would produce respectable numbers from an explanation that had been reduced to noise. Any such leak would invalidate every comparison that follows.

We therefore replace each explanation with a null version and rescore it. Two nulls are used, defined in \cref{sec:experiments}: the permuted null of Eq.~\ref{eq:nullperm} shuffles each curve's values across its grid, which leaves every effect magnitude intact and destroys only the mapping from feature value to effect; the randomised null of Eq.~\ref{eq:nullrand} replaces each curve with amplitude-matched noise. We also add noise to the intact curves in increasing amounts, in units of each curve's own amplitude, so that the degradation is graded rather than all-or-nothing. Every number in this subsection uses the parameter-free surrogate $g_{\mathbf 1}$ of Eq.~\ref{eq:unit}, scored by $R^2$ on the held-out split against the black box's predictions; the fitted variants are excluded here because rescaling the curves could partly repair a damaged explanation and blunt the collapse we are trying to expose.

\begin{table}[htbp]
\centering
\small
\caption{Test $R^2$ against the black box under the unit-coefficient construction $g_{\mathbf 1}$ of Eq.~\ref{eq:unit}, averaged over datasets and model families. \emph{Intact} is the unmodified explanation, \emph{Values permuted} the null of Eq.~\ref{eq:nullperm}, and \emph{Randomised} the null of Eq.~\ref{eq:nullrand}. \emph{Collapse} is the difference between the intact and permuted columns. Parenthesised ranges are 95\% intervals.}
\label{tab:controls}
\begin{tabular}{lrrrr}
\toprule
Method & Intact & Values permuted & Randomised & Collapse \\
\midrule
PDP & 0.702 \,(0.695, 0.709) & -0.654 \,(-0.678, -0.630) & -0.617 \,(-0.636, -0.597) & 1.357 \\
ALE & 0.713 \,(0.706, 0.720) & -0.741 \,(-0.756, -0.726) & -0.693 \,(-0.720, -0.667) & 1.454 \\
SHAP & 0.798 \,(0.794, 0.802) & -0.416 \,(-0.434, -0.399) & -0.380 \,(-0.396, -0.365) & 1.215 \\
LIME & 0.540 \,(0.535, 0.544) & -0.218 \,(-0.232, -0.204) & -0.196 \,(-0.207, -0.184) & 0.758 \\
\bottomrule
\end{tabular}
\end{table}

\begin{figure}[htbp]
\centering
\includegraphics[width=0.62\linewidth]{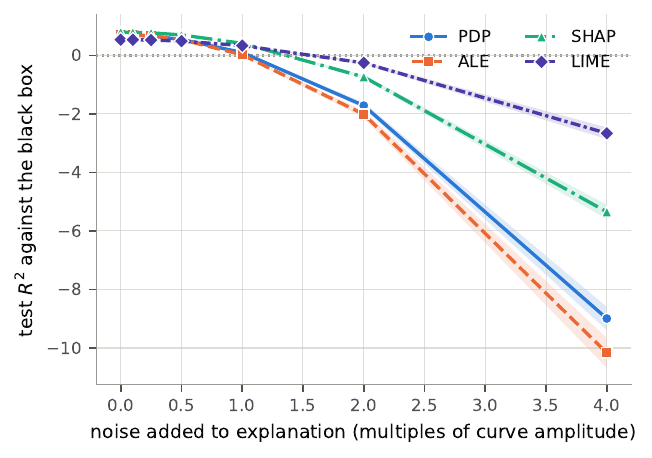}
\caption{Test $R^2$ against the black box as Gaussian noise of increasing
standard deviation is added to each explanation's curves, in units of that curve's own amplitude. Bands are 95\% intervals over repeated splits.}
\label{fig:corruption}
\end{figure}

\Cref{tab:controls} shows the intact and null scores for each method, and \cref{fig:corruption} shows the graded version. The collapse is complete: PDP falls from $\numPdpIntact$ to $\numPdpPerm$ and SHAP from $\numShapIntact$ to $\numShapPerm$. Negative values are worse than predicting the model's average output, so a null explanation is not merely uninformative but actively misleading about the model. \Cref{fig:corruption} shows the same thing without a cliff edge: fidelity falls monotonically with added noise for all four methods, crossing zero between one and two curve amplitudes.

The metric therefore passes the test it had to pass. Nothing in the construction lets an explanation score well without carrying information about how feature values map to model output, which is the property every comparison in the rest of this section relies on.

Two further readings of \cref{tab:controls} are worth taking while it is in front of us. The \emph{Intact} column already orders the methods SHAP (best explanations) $>$ ALE $>$ PDP $>$ LIME (poorest explanations), which \cref{sec:results-main} confirms with paired tests. And the \emph{size} of the collapse differs sharply: LIME falls by $0.758$ where ALE falls by $1.454$. That is a fact about scale rather than about information. LIME's local models are ridge-regularised, so its coefficients are attenuated and its curves have small amplitude; the nulls preserve amplitude by construction, and summing near-flat wrong curves approximates predicting the mean, which scores near zero rather than far below it. The collapse is thus confounded with how boldly a method states its case, which is why we report the intact value and the collapse separately rather than combining them.

\subsection{Testing the refitting-gain prediction}\label{sec:results-theory}

\Cref{cor:refit} makes a prediction that can be checked: when features are independent, fitting coefficients to the curves cannot improve on leaving them at unity, and \cref{cor:diagnostic} adds that any gain which does appear is a signal of feature dependence. It is important to test this, for two reasons. The prediction is what licenses the parameter-free construction used throughout the paper, so if refitting turned out to help substantially under independence the construction would be poorly chosen. And if the gain does track dependence, it delivers the practical diagnostic argued for in \cref{sec:theory}: a check on the independence condition that costs one extra least-squares fit and never requires estimating the feature distribution.

Testing it requires knowing the true feature correlation, which real data cannot supply. We therefore use the synthetic suite, where interaction strength $\tau$ and feature correlation $\rho$ are set independently by construction. Each design is drawn \numNDraws{} times from independent seeds and the draws are averaged, so that a difference between designs reflects the design rather than the particular sample each happened to receive. We report the relative gain $\Delta^{\mathrm{rel}}_{\mathrm{OLS}}$ of Eq.~\ref{eq:gain}, because the designs differ in overall fidelity and the same absolute gap means different things on different rows.

\begin{table}[htbp]
\centering
\small
\caption{Relative refitting gain for PDP, $(R^2(g_{\hat\alpha}) - R^2(g_{\mathbf 1}))/R^2(g_{\mathbf 1})$ of Eq.~\ref{eq:gain}, in per cent. Rows are the synthetic designs, those with independent features ($\rho = 0$) above the rule and those with correlated features below; $\rho$ is the designed feature correlation of Eq.~\ref{eq:dgpX}. Columns are the four model families. Each entry averages the independent draws of that design.}
\label{tab:refit-gain}
\begin{tabular}{llrrrr}
\toprule
Design & $\rho$ & Gradient Boosting & MLP & Random Forest & SVM RBF \\
\midrule
syn\_inter00 & 0.00 & +0.1 & +0.0 & -1.9 & -0.0 \\
syn\_inter10 & 0.00 & +0.2 & +0.1 & -1.5 & +0.1 \\
syn\_inter25 & 0.00 & +0.3 & +0.0 & -2.0 & +0.0 \\
syn\_inter50 & 0.00 & +0.2 & -0.3 & -2.2 & -0.4 \\
syn\_inter75 & 0.00 & -0.1 & +0.3 & -4.1 & -0.3 \\
\midrule
syn\_corr30 & 0.30 & +1.4 & +2.4 & +0.9 & +2.0 \\
syn\_corr60 & 0.60 & +4.8 & +18.0 & +4.1 & +7.2 \\
syn\_corr85 & 0.85 & +7.1 & +59.0 & +3.9 & +38.0 \\
syn\_regress & 0.20 & +1.8 & +1.4 & +1.1 & +1.2 \\
\bottomrule
\end{tabular}
\end{table}

\Cref{tab:refit-gain} shows the relative gain for PDP on each design. In the independent block it is within $\pm 0.4\%$ of zero for Gradient Boosting, the MLP and the SVM RBF at every interaction strength. In the correlated block it grows with $\rho$: roughly $1\%$--$2\%$ at $\rho = 0.30$, $4\%$--$18\%$ at $\rho = 0.60$, and $4\%$--$59\%$ at $\rho = 0.85$. Random Forest is offset downwards throughout the independent block, sitting at $-1.5\%$ to $-4.1\%$ where the other three sit at zero.

Both halves of the prediction hold. Refitting buys nothing under independence, whatever the interaction strength, which is what \cref{cor:refit} requires and what justifies using $g_{\mathbf 1}$ as the paper's metric. And the gain rises with feature correlation, confirming \cref{cor:diagnostic}'s claim that it carries information about dependence.

Two limits on its use follow from the same table. The random forest's negative offset under independence means the null is model-dependent, so a value must be calibrated per model family before it can be called positive, exactly the caveat attached to \cref{cor:diagnostic}. And the spread across families at high correlation is wide: at $\rho = 0.85$ the MLP gains $59\%$ and Random Forest $3.9\%$ from the same data. The gain detects dependence; it does not measure it on a scale comparable across models.

\paragraph{Why the tree ensembles gain so much less.}
The split at $\rho = 0.85$ is between the two tree ensembles, which gain $7.1\%$ and $3.9\%$, and the two smooth models, which gain $59\%$ and $38\%$. It is not an artefact of the relative scale: the absolute gains differ in the same direction and by a similar factor, $0.062$ and $0.035$ for Gradient Boosting and Random Forest against $0.272$ and $0.175$ for the MLP and the SVM RBF. What drives it is how much the parameter-free surrogate had already achieved. At $\rho = 0.85$ the unit surrogate reproduces the tree ensembles well ($R^2 = 0.885$ and $0.905$) and the smooth models poorly ($0.541$ and $0.623$), so there is far more for refitting to recover in the smooth case, and refitting duly recovers it, bringing all four to between $0.80$ and $0.95$.

We offer the following explanation, and flag that we have not tested it directly. A partial dependence function averages $\fhat$ over the \emph{marginal} distribution of the remaining features, so at $\rho = 0.85$ it evaluates the model at feature combinations that the joint distribution almost never produces. How badly the resulting curve is distorted therefore depends on how the model behaves away from the data. A tree ensemble is piecewise constant on axis-aligned cells and its predictions off the data manifold are bounded by the leaf values it learned on it, so its partial dependence curves stay close to the right scale. An RBF support vector machine and a multilayer perceptron are smooth global functions that extrapolate, and off-manifold they can return systematically inflated or deflated values, which distorts the amplitude of their curves. A single coefficient per feature is exactly the correction such a distortion needs, which is why refitting buys them so much. Testing this would require comparing on-manifold and off-manifold partial dependence estimates directly, which we did not do.

For the parameter-free surrogate the practical reading is favourable. At the correlations typical of real tabular data the cost of fixing the coefficients at unity is $1\%$--$2\%$ of $R^2$, so $g_{\mathbf 1}$ remains a reliable instrument even though the independence condition that formally licenses it does not hold. Only at extreme correlation does that change, and there it changes sharply for the smoother models.

\subsection{Is it the explanation or the aggregator?}\label{sec:results-aggregator}

Comparing SHAP and LIME against PDP and ALE through a shared construction is only fair if the construction itself is not doing the work. There are two specific concerns. First, SHAP and LIME receive an estimation step that PDP and ALE do not: their per-observation attributions must be smoothed into curves by the binning of Eq.~\ref{eq:binned}, and that step is a conditional average. A conditional average can track $\E[\fhat \mid X_j]$ under feature dependence in a way a marginal partial dependence curve provably cannot, so it could hand SHAP and LIME an advantage that has nothing to do with their attributions. Second, SHAP's tree estimator has a marginal and a conditional variant, the choice between them is usually left at a library default, and the two estimate different quantities, so a SHAP result could be an artefact of that default.

We address the first concern with a control that removes the explanation entirely: the black box's own predictions are passed through the identical binning, giving curves $\hat c_j(x_j) \approx \widehat{\E}[\fhat \mid X_j = x_j]$ with no attribution anywhere in the pipeline. If the aggregator were doing the work, this control would match the explanation-derived predictors. For the second we compute both TreeSHAP variants separately on the two tree families where both are defined. All results here are on the real datasets and use the refitted construction $g_{\hat\alpha}$ of Eq.~\ref{eq:ols}; refitting is used because the control sums $p$ conditional means that each carry the whole marginal signal, and under unit coefficients it would over-count by roughly a factor of $p$, producing a scaling artefact rather than an informative comparison.

\begin{table}[htbp]
\centering
\small
\caption{Test $R^2$ against the black box under the refitted construction $g_{\hat\alpha}$ of Eq.~\ref{eq:ols}, averaged over the real datasets. Rows are curve sources, columns model families. \emph{SHAP, conditional TreeSHAP} replaces the marginal \texttt{interventional} estimator with \texttt{tree\_path\_dependent} and is defined only on the two tree families. \emph{E[f\,$\vert$\,$x_j$], binned} applies the binning of Eq.~\ref{eq:binned} to the black box's own predictions, with no explanation in the pipeline. Dashes mark undefined combinations.}
\label{tab:aggregator-control}
\begin{tabular}{lrrrr}
\toprule
Curve source & Gradient Boosting & MLP & Random Forest & SVM RBF \\
\midrule
PDP & 0.799 & 0.816 & 0.848 & 0.756 \\
ALE & 0.801 & 0.816 & 0.848 & 0.759 \\
SHAP & 0.814 & 0.837 & 0.857 & 0.784 \\
SHAP, conditional TreeSHAP & 0.813 & -- & 0.859 & -- \\
E[f | x\_j], binned (no explanation) & 0.769 & 0.788 & 0.824 & 0.745 \\
\bottomrule
\end{tabular}
\end{table}

\Cref{tab:aggregator-control} shows every curve source against every model family. The control is the last row. It is beaten by all three curve sources on all four model families, reaching $\numOlsCondMean$ on average against $\numOlsPDP$ for PDP, $\numOlsALE$ for ALE and $\numOlsSHAP$ for SHAP; each of the three beats it at the dataset level (PDP $p = \numPDPVsCondP$, ALE $p = \numALEVsCondP$, SHAP $p = \numSHAPVsCondP$). The two TreeSHAP variants sit on top of each other: tested over the two tree families the difference is $\numTreeShapDiffUnit$ under the unit construction ($p = \numTreeShapPUnit$) and $\numTreeShapDiffOls$ under refitting ($p = \numTreeShapPOls$).

Neither concern survives. The shared aggregator does not reproduce the explanations (applied on its own it loses to every one of them), so the comparison in the next subsection is between explanations and not between smoothers. And the estimator choice does not account for the SHAP result, since the marginal and conditional variants are indistinguishable on this benchmark. The advantage lies in the attributions.

\subsection{Comparing the reliability of the four explanation methods}\label{sec:results-main}

With the construction validated and the aggregator ruled out, the comparisons in the rest of the paper can be read directly. Two questions matter to a practitioner: which explanation method carries most of the model, and how much of the model any of them carries relative to using no explanation at all. The second matters because a ranking is only useful if the quantities being ranked are worth having.

We score all four methods on the real datasets under the three constructions of Eqs.~\ref{eq:unit}--\ref{eq:ridge}, test the six pairwise differences with the dataset as the unit of replication and Holm correction over the family, and place the results against four reference points scored on the original outcome: an intercept-only floor, a linear model on the raw features, an additive spline model fitted directly to $y$, and the average of the four black-box models (Gradient Boosting, MLP, Random Forest and SVM RBF).

\begin{table}[htbp]
\centering
\small
\caption{Test $R^2$ of each explanation-derived predictor against the black box's own predictions. Rows are the constructions of Eqs.~\ref{eq:unit}--\ref{eq:ridge} --- $g_{\mathbf 1}$, $g_{\hat\alpha}$ and $g_{\hat\alpha_\lambda}$ --- plus LIME's native local-IDW variant; columns are the explanation methods. Entries are test $R^2$, averaged over datasets and over the four model families, with the half-width of a 95\% $t$ interval. A dash marks a construction a method does not have.}
\label{tab:fidelity-compact}
\begin{tabular}{lrrrr}
\toprule
Construction & PDP & ALE & SHAP & LIME \\
\midrule
additive, unit & 0.702 $\pm$ 0.007 & 0.713 $\pm$ 0.007 & 0.798 $\pm$ 0.004 & 0.540 $\pm$ 0.005 \\
additive, ols & 0.805 $\pm$ 0.005 & 0.806 $\pm$ 0.005 & 0.823 $\pm$ 0.005 & 0.762 $\pm$ 0.005 \\
additive, ridge & 0.806 $\pm$ 0.005 & 0.806 $\pm$ 0.005 & 0.824 $\pm$ 0.005 & 0.762 $\pm$ 0.005 \\
local-IDW & -- & -- & -- & 0.463 $\pm$ 0.034 \\
\bottomrule
\end{tabular}
\end{table}
\begin{table}[htbp]
\centering
\small
\caption{Paired differences in test $R^2$ against the black box between explanation methods, under the unit-coefficient construction $g_{\mathbf 1}$ of Eq.~\ref{eq:unit}. The entry in row $a$, column $b$ is $\Delta R^2 = R^2(a) - R^2(b)$, followed in parentheses by the Holm-corrected $p$-value of a paired $t$-test over datasets. A positive entry favours the row method. The unit of replication is the dataset ($n = 13$): differences are averaged to one value per dataset before testing. Holm correction is over the 6 comparisons. The lower triangle is left empty because it carries the same numbers with the sign reversed.}
\label{tab:pairwise}
\begin{tabular}{lrrrr}
\toprule
 & PDP & ALE & SHAP & LIME \\
\midrule
PDP & -- & -0.011 \, (0.781) & -0.096 \, (0.019) & +0.163 \, (0.019) \\
ALE & -- & -- & -0.085 \, (0.019) & +0.173 \, (1e-05) \\
SHAP & -- & -- & -- & +0.258 \, (4e-05) \\
LIME & -- & -- & -- & -- \\
\bottomrule
\end{tabular}
\end{table}

\Cref{tab:fidelity-compact} shows the test $R^2$ of each method and construction, and \cref{tab:pairwise} shows the paired differences. SHAP is first and LIME last: SHAP leads PDP by $\numShapPdpDiff$ and ALE by $\numShapAleDiff$, and every method leads LIME. PDP and ALE are not separated, differing by only $\numAlePdpDiff$.

That last result is worth stating positively rather than as a failure to reject. ALE was designed to behave better than PDP under feature dependence, and these are dependent real datasets, yet on this criterion and at this sample size the two cannot be told apart.

One explanation is that the two properties pull in opposite directions and roughly cancel. ALE's advantage over PDP is confined to dependence: it accumulates local effects computed within narrow intervals of the feature, so it never evaluates the model at the off-manifold combinations that distort a partial dependence curve. But that same construction is awkward when the model contains interactions, since the local effect of $x_j$ then depends on where the other features happen to sit within each interval, and accumulating those effects into a single curve has no clean interpretation, a limitation Apley and Zhu~\cite{apley2020visualizing} are explicit about. Real tabular data supplies both dependence and interaction at once, so ALE gains on one axis and loses on the other. That would produce exactly what we observe: a difference close to zero, rather than the ALE advantage the dependence argument alone would predict. We put this forward as a reading consistent with the data rather than as something these experiments isolate, since separating the two effects would need designs that vary dependence and interaction independently and then measure ALE against PDP on each; our synthetic suite varies them independently but was not analysed for this contrast. \Cref{tab:fidelity-compact} also shows why LIME's position depends on which construction is read: its unit and refitted scores differ by more than $0.2$, where SHAP's differ by $0.025$. LIME's information is present but mis-scaled, because its ridge-regularised local coefficients are attenuated; reporting either construction alone would misrepresent it in one direction or the other.

\begin{table}[htbp]
\centering
\small
\caption{Predictors scored against the original target $y$ (test ROC-AUC), averaged over datasets and over the four model families, with a 95\% interval. Rows are ordered best to worst. Explanation rows use the unit-coefficient construction $g_{\mathbf 1}$ of Eq.~\ref{eq:unit}; \emph{Additive spline model} and \emph{Linear on raw features} are fitted directly to $y$ with no explanation involved.}
\label{tab:baselines}
\begin{tabular}{lr}
\toprule
Predictor & Score \\
\midrule
Black box & 0.917 \,(0.913, 0.921) \\
Additive spline model (fitted to $y$) & 0.900 \,(0.896, 0.904) \\
SHAP (additive, unit) & 0.895 \,(0.890, 0.899) \\
PDP (additive, unit) & 0.890 \,(0.885, 0.894) \\
Linear on raw features & 0.883 \,(0.880, 0.887) \\
ALE (additive, unit) & 0.870 \,(0.866, 0.875) \\
LIME (additive, unit) & 0.845 \,(0.841, 0.849) \\
Intercept only & 0.500 \,(0.500, 0.500) \\
\bottomrule
\end{tabular}
\end{table}

\Cref{tab:baselines} places these numbers on the outcome axis, and three things follow from it. First, the ordering of the explanation methods is the same one \cref{sec:results-validity} found against the black box: SHAP is best and LIME worst, with PDP and ALE between them. That the ranking survives a change of target (from reproducing the model to recovering the outcome) is worth noting, since the two are different questions and need not have agreed.

Second, the explanation-derived predictors come close to the black box itself. SHAP's unit surrogate reaches $\numYSHAP$ against $\numYBlackBox$ for the average of the four black-box models, so a predictor built from nothing but an explanation recovers most of what the model achieves. Note where the linear baseline falls: at $\numYLinear$ it is not below all four explanations but among them, above ALE and LIME and below SHAP and PDP. The floor against which the explanations should be read is therefore the intercept-only row at $\numYIntercept$; the linear model is better understood as a fifth competitor than as a baseline.

Third, and most informative, the explanations are essentially level with the additive spline model at $\numYSpline$, a model fitted directly to $y$ with no explanation involved, and the best any additive function of the features can do here. The explanation curves therefore capture very nearly all of the marginal information available. This is what \cref{prop:projection} leads one to expect: the curves span the additive space, so a predictor built from them should reach the additive optimum and stop there. The remaining gap to the black box is the interaction structure, which no additive explanation can convey.

The practical conclusion is therefore narrower than the ranking alone suggests. The ordering is real and reliably measured, but on the axis a practitioner cares about (recovering the outcome), the four methods differ little, and none of them approaches the black box. The value of the metric is in auditing how much of a model an explanation conveys, not in choosing an explanation method to deploy as a predictor.

\paragraph{The ranking is a dependence phenomenon.}
The ordering above is not a general property of the methods, and the synthetic suite shows why. Where features are independent by construction, \cref{prop:projection} says the partial dependence surrogate is already the optimal additive approximation, so nothing should beat it. That is what happens: on the \numNIndep{} independent-feature designs SHAP is \emph{worse} than PDP ($\numShapPdpIndep$, winning on $\numShapPdpIndepWins$ designs, $p = \numShapPdpIndepP$) and indistinguishable from ALE. An attribution-based estimate of an object PDP already computes exactly can only add noise. On the \numNReal{} real datasets, where features are dependent and the projection property lapses, SHAP leads PDP by $\numShapPdpReal$ ($\numShapPdpRealWins$, $p = \numShapPdpRealP$). Both figures come from the same per-cell results that \cref{tab:fidelity-compact,tab:pairwise} summarise, split by whether the dataset's features are independent by construction rather than pooled; the real-data figure is the split that \cref{tab:pairwise} reports, and the independent-feature figure uses the synthetic designs that the rest of this section sets aside. So the ranking is what happens once the condition that makes PDP optimal is violated, which on real tabular data it almost always is. 

\subsection{Relation to established metrics}\label{sec:refmetrics}

A new evaluation metric has to answer what it adds. If it merely reproduces the ranking of an existing metric then it is a reformulation, however differently motivated, and a practitioner already computing that metric gains nothing by adopting it.

We therefore compute five established explanation metrics (infidelity~\cite{yeh2019infidelity}, faithfulness correlation~\cite{bhatt2020evaluating}, max-sensitivity~\cite{yeh2019infidelity}, complexity~\cite{bhatt2020evaluating} and sparseness~\cite{chalasani2020concise})
on the same explanations, the same models and the same splits as our own, and correlate the rankings. Each has free parameters, and a rank correlation is only meaningful if they are stated: infidelity uses $50$ Gaussian perturbations at scale $0.2$ standard deviations; faithfulness correlation uses $50$ random subsets covering $30\%$ of the features, replaced by the training mean; max-sensitivity uses $20$ perturbations at radius $0.1$ standard deviations; complexity and sparseness have no parameters beyond the attribution matrix. All five are evaluated on the same $200$ held-out points. We took these values from the defaults in the originating papers and did not tune them.

\begin{table}[htbp]
\centering
\small
\caption{Spearman rank correlations between the proposed metric (test $R^2$ against the black box, unit-coefficient construction) and five established metrics computed on the same explanations and splits, one row per established metric. \emph{Pooled $\rho$} ranks over all (dataset, model, method, split) cells at once, with its $p$-value alongside. \emph{Within-cell $\rho$} ranks the four explanation methods against each other inside a single (dataset, model, split) cell and averages over the 1040 such cells, with a 95\% interval.}
\label{tab:refmetrics}
\begin{tabular}{lrrr}
\toprule
Established metric & Pooled $\rho$ & $p$ & Within-cell $\rho$ \\
\midrule
infidelity & 0.120 & 7.9e-15 & 0.433 \,(0.403, 0.463) \\
faithfulness corr & 0.628 & 0.0e+00 & 0.452 \,(0.422, 0.481) \\
max sensitivity & -0.237 & 2.5e-54 & 0.035 \,(-0.001, 0.071) \\
complexity & -0.090 & 6.4e-09 & 0.003 \,(-0.033, 0.038) \\
sparseness & 0.190 & 3.2e-35 & 0.012 \,(-0.024, 0.047) \\
\bottomrule
\end{tabular}
\end{table}

\Cref{tab:refmetrics} shows the correlations two ways. Pooled across cells, the metric is close to orthogonal to infidelity ($\rho = \numRhoInfid$), complexity ($\numRhoCplx$) and sparseness ($\numRhoSpars$), and correlated with faithfulness correlation ($\numRhoFaith$) and weakly with max-sensitivity ($\numRhoSens$). Within a (dataset, model, split) cell, the comparison in which a choice between explanation methods is actually taken, the picture changes: faithfulness correlation $\numRhoWithinFaith$, but also infidelity $\numRhoWithinInfid$. Sparseness remains indistinguishable from zero, and so does complexity ($\numRhoWithinCplx$, with an interval of $(\numRhoWithinCplxLo, \numRhoWithinCplxHi)$ that includes zero).

The honest summary is that the metric belongs to the \emph{faithfulness family} rather than restating any single member of it. It is best understood as a global, out-of-sample, construction-based relative of faithfulness correlation, and its value does not rest on being orthogonal to it. What it adds is the decomposition: a fidelity number paired with $\Radd$ separates "this explanation is poor'' from "this model is not additive, so no additive explanation could have done better'', a distinction no single faithfulness score makes.

\paragraph{Metric versus mechanism.}
A low correlation could arise from either of two sources and this design does not separate them. One is the \emph{scoring function}: ours is an $R^2$ against the black box, theirs are perturbation-based quantities on other scales. The other is the \emph{evaluation mechanism}: ours is the only one that builds a predictor from the explanation and tests it out of sample, whereas the other five interrogate the attribution matrix in place. We hold the reference metrics at the mechanism their authors specify and do not compute, say, an infidelity of our surrogate, so the comparison is between the metrics as they are used rather than a factorial decomposition. Our reading is that the mechanism is the larger part of the difference (an out-of-sample predictive test can fail in ways an in-place perturbation test cannot see, which is what the next subsection exploits), but the present experiment does not isolate it.

\subsection{Is the proposed metric better?}\label{sec:results-discrimination}

The previous subsection established that the metric is not a restatement of an existing one. That is a claim about difference, and it leaves the harder question open: is it \emph{better}? This is the question a practitioner choosing a metric actually faces, and it is rarely asked of explanation metrics, which are usually justified by the reasonableness of their definition rather than tested against a known answer.

It can be asked, because any evaluation metric asserts an ordering: given two explanations of the same model, it says which is better. That assertion is testable whenever we can manufacture pairs whose true ordering is known. We build such pairs by degrading an explanation by a controlled amount, giving a ladder of progressively worse versions,
\begin{equation}\label{eq:ladder}
  E \;\succ\; E_{\sigma = 0.5} \;\succ\; E_{\sigma = 1} \;\succ\;
  E_{\sigma = 2} \;\succ\; E_{\mathrm{perm}},
\end{equation}
where $E_\sigma$ adds Gaussian noise of standard deviation $\sigma$ times each curve's own amplitude (Eq.~\ref{eq:corrupt}) and $E_{\mathrm{perm}}$ is the permuted null of Eq.~\ref{eq:nullperm}. The ordering is fixed by construction: each rung is the previous explanation plus strictly more noise. For every (dataset, model, split, method) on the real data we compute all six evaluation metrics at every rung and record whether each ranked the intact explanation above the degraded one, reading each metric in the direction its own authors specify. Aggregating with the dataset as the unit gives a probability of correct ranking, for which $0.5$ is chance.

Difficulty comes from degrading the explanation rather than from noising the data. Noising the data would change what the true explanation is, so the ground-truth ordering the comparison rests on would itself become an estimate, and a metric could be penalised for correctly tracking an ordering we had mis-specified. Degrading the explanation leaves the ground truth exact, and the intermediate rungs supply difficulty by sitting close to the intact explanation.

\begin{table}[htbp]
\centering
\small
\caption{One row per evaluation metric. Columns are the rungs of the degradation ladder of Eq.~\ref{eq:ladder}: $\sigma$ is Gaussian noise added to the explanation's curves in units of each curve's own amplitude, and \emph{permuted} is the null of Eq.~\ref{eq:nullperm}. Each entry is the fraction of (model, split, explanation method) cells in which the metric ranked the intact explanation above that rung, averaged over the 13 datasets; \emph{Mean} averages across rungs. Each metric is read in the direction its originating paper specifies, so that a larger entry always means better agreement with the known ordering. Rows are ordered by \emph{Mean}, best first.}
\label{tab:discrimination}
\begin{tabular}{lrrrrr}
\toprule
Evaluation metric & $\sigma = 0.5$ & $\sigma = 1$ & $\sigma = 2$ & permuted & Mean \\
\midrule
Faithfulness correlation & 0.940 & 0.987 & 0.994 & 0.998 & 0.980 \\
Proposed ($R^2$ vs.\ $\hat f$) & 0.932 & 0.962 & 0.981 & 0.995 & 0.967 \\
Max-sensitivity & 0.957 & 0.959 & 0.961 & 0.956 & 0.958 \\
Infidelity & 0.757 & 0.887 & 0.962 & 0.347 & 0.738 \\
Sparseness & 0.368 & 0.361 & 0.364 & 0.291 & 0.346 \\
Complexity & 0.341 & 0.333 & 0.338 & 0.283 & 0.324 \\
\bottomrule
\end{tabular}
\end{table}

\Cref{tab:discrimination} shows the accuracy of each metric at each rung. Three patterns stand out, and they are of different kinds. Faithfulness correlation ranks the ladder correctly \numDiscFaithMean{} of the time against \numDiscProposedMean{} for ours and \numDiscSensMean{} for max-sensitivity. Taking the dataset as the unit, the difference between ours and faithfulness correlation is $\numDiscVsBestDiff$ ($p = \numDiscVsBestP$, ours ahead on $\numDiscVsBestWins$ of \numDiscNDatasets{} datasets): the two cannot be separated on this evidence.

Infidelity is defeated by the sharp null. It handles noise well, rising from $0.757$ at $\sigma = 0.5$ to $0.962$ at $\sigma = 2$, but against the permuted explanation it scores \numDiscInfidPerm{}, far below chance: it prefers the permuted version. Permutation preserves the multiset of curve values exactly and only reassigns them to different feature values, so attribution magnitudes survive intact, and infidelity's local perturbation test is largely a test of magnitude.

Complexity and sparseness are inverted. Both score below chance at every rung, averaging \numDiscCplxMean{} and \numDiscSparsMean{}, meaning they systematically prefer the degraded explanation. Neither is a faithfulness metric: both quantify how \emph{concentrated} an attribution vector is, and concentration is a property an explanation can have while being entirely wrong. Adding noise to a curve tends to inflate whichever feature the noise happened to favour, which concentrates the attributions and improves both scores.

Returning to the question the subsection opened with: our metric is not better than the best established alternative, and we do not claim it is. What it is, is a member of the group that works: it detects a degraded explanation as reliably as faithfulness correlation does, and the case for the framework rests on the decomposition of \cref{sec:refmetrics} rather than on this comparison. The more consequential finding concerns the others. Of the five established metrics tested, two track the ordering reliably, one fails on the single most important case, and two point the wrong way. Benchmark suites that report complexity and sparseness alongside faithfulness metrics invite the reading that a better score means a better explanation, and on this evidence that reading is reversed.

\paragraph{What this experiment can and cannot settle.}
The ladder degrades \emph{curves}, and three of the six metrics are computed from attributions derived from those curves, so there is a structural affinity between the test and part of what is being tested; a ladder built by corrupting attributions before they are smoothed might order the metrics differently. Where the experiment is unambiguous is in the negative direction: a metric that ranks a permuted explanation above an intact one is not measuring faithfulness, whatever else it may be measuring, and no affinity argument rescues it.

\section{Discussion}\label{sec:limitations}

\paragraph{Whether the black box was warranted at all.}
Rudin~\cite{rudin2019stop} argues that for high-stakes decisions one should use an inherently interpretable model rather than explain an opaque one. Our results speak to that claim without settling it, and they lean the other way.

Two facts from \cref{tab:baselines} matter. The black box is the most accurate predictor in the study, reaching AUC $\numYBlackBox$ against $\numYSpline$ for the spline GAM fitted directly to $y$, so the interpretable model does carry a real accuracy cost here. And the explanation-derived predictors come within a hair of the GAM (SHAP's reaches $\numYSHAP$), so the explanations convey very nearly as much of the task as the interpretable model does in total. On this evidence the combination of an accurate black box with model-agnostic explanations performs well: one keeps the accuracy and gives up little of what the interpretable model would have shown.

The obvious rejoinder is that an explanation is only worth what it faithfully reports, and a predictor built from an unfaithful explanation could score well for the wrong reasons. This is where the theory does some work. \Cref{prop:projection} establishes that under feature independence the partial dependence surrogate \emph{is} the optimal additive summary of the model, so there is a regime in which these explanations are provably as good as any additive description could be. Outside it, \cref{cor:diagnostic} identifies the sole source of degradation as feature dependence, and \cref{sec:results-theory} shows the resulting cost is $1\%$--$2\%$ of $R^2$ at the correlations typical of tabular data. So the reliability of the explanations is not merely assumed: it is bounded by a stated condition and the departure from it is measurable.

None of this settles the debate, which turns on considerations (auditability, recourse, contestability) that a fidelity number does not reach. What our framework contributes is a way to make one side of it quantitative: the cost of using a black box plus explanations, rather than an interpretable model, can now be read off as the gap between the black box and the explanation-derived predictor, and on these datasets that gap is small.

\paragraph{What the metric does not capture.}
An explanation can be informative without being comprehensible, and our construction is blind to the difference. A method that scattered predictive content across forty features would score well here and be useless to a practitioner. The metric belongs alongside complexity and robustness criteria, not in place of them; its weak association with the complexity family (\cref{sec:refmetrics}) suggests it measures a largely different axis, not that the other axis is unimportant.

\paragraph{The additive form is a choice, and so is the one-dimensional curve.}
Both curve-based methods and our SHAP/LIME conversions produce additive surrogates, so the framework cannot reward an explanation for conveying interaction structure. This is a deliberate restriction (a more flexible surrogate would measure the surrogate's capacity rather than the explanation's content), but it bounds what the comparison can say.

The restriction has two layers, and they are worth separating. The first is the \emph{summation}: we add the per-feature curves rather than combining them in some richer way. The second is the \emph{input}: we take one-dimensional curves, although PDP and ALE are defined for feature subsets of any size and two-dimensional versions are routinely used to inspect pairwise interactions. Admitting second-order curves would raise the ceiling, since a surrogate $a_0 + \sum_j a_j(x_j) + \sum_{j<k} a_{jk}(x_j,x_k)$ can reach the order-two Hoeffding components that $\mathcal{A}$ of Eq.~\ref{eq:addspace} cannot, 
and \cref{prop:projection} generalises to that space unchanged. What does not generalise cheaply is the parameter-free property: the number of second-order terms grows as $p^2$, most of them are near zero, and estimating which to keep reintroduces exactly the fitting step whose absence makes the present metric attributable. Extending the framework in this direction while keeping it parameter-free is open, and is the most obvious next step.

\paragraph{Output scale.}
Classification is explained on the probability scale. Log-odds is arguably more principled, since additivity is a property of the score, but it is not safe by default: a tree ensemble grows pure leaves and emits probabilities of exactly $0$ and $1$, so the required clip becomes the dominant determinant of the score's variance. With 60\% of predictions saturated on one dataset we measured $\Radd = -2.11$, an artefact entirely of the clip. Setting the clip from the ensemble's own resolution ($1/2T$ for $T$ trees) removes most but not all of the damage. Practitioners running this framework on log-odds should report the saturated fraction.

\paragraph{Library behaviour.}
Two implementation details silently invert results and are worth recording. In regression mode the reference LIME implementation stores its fitted coefficients under one label key and a \emph{sign-negated} display copy under another. Reading the first key (the obvious choice) inverts every local model, which turns a positive fidelity into a negative one: the surrogate then predicts the black box's deviations backwards, so it scores worse than predicting the mean. We caught this during development and it is now guarded by a unit test that fits LIME to a black box known to be monotone increasing in one feature and asserts the recovered coefficient is positive.
Separately, TreeSHAP's \texttt{raw} output for a scikit-learn tree classifier is the probability, not the log-odds, so on the log-odds scale its attributions violate local accuracy and every construction built on that decomposition becomes invalid. Both are guarded by assertions in our released code.

\paragraph{What excluding \texttt{magic\_telescope} costs.}
We removed one real dataset from the study because a single (dataset, model) cell, the RBF support vector machine, produced a partial dependence reconstruction far worse than predicting the model's own mean, and that cell dominated every pooled average it entered. The decision is defensible but it is not free, and it cuts against the paper's own argument in one respect: the regime in which partial dependence fails catastrophically is exactly the regime \cref{prop:projection} says to watch for, so removing the clearest real instance of it removes evidence for our own thesis as well as noise. What remains is the controlled correlation sweep, where the same failure is produced deliberately and can be read against a known $\rho$. A study aimed specifically at characterising catastrophic PDP failure should keep such datasets and report them per model rather than pooled; ours is aimed at comparing explanation methods on average, and for that purpose one dominating cell is a liability. We flag the exclusion here so that it is not mistaken for a data-cleaning step.

\paragraph{Reproducibility.}
All code, dataset specifications, and the full result table are released. Every table and figure is generated from the recorded results and read at compile time. So is every number quoted in the prose: each is emitted by the same pipeline as a macro and referenced by name, so the text cannot disagree with the tables beside it. We adopted this after finding that it had already happened: when one run superseded another, several sentences kept the earlier run's values while the adjacent tables were regenerated. Hand-checking is what failed, so hand-checking is not the remedy.

Explanations that claim to describe how a model works can be asked to predict. Turning that observation into a metric requires care: the construction must be shared across explanation families to be fair, parameter-free to be attributable, and interpreted against a reference that says how much additive structure was available to recover in the first place. We have made explicit the identification of that reference with the $L^2$ additive projection under independence, and shown that the condition is not a technicality: it is what decides the empirical ranking.

Across \numNReal{} real datasets and \numNSynthDesign{} synthetic designs, SHAP-derived curves carry the most recoverable predictive content and LIME the least, with ALE and PDP not reliably separated. But the ordering is entirely a dependence phenomenon. Where features are independent by construction, PDP is provably optimal and is measurably not beaten; where they are dependent (which on real tabular data they invariably are), SHAP gains, and two controls establish that the gain is in the attributions rather than in the aggregation used to read them. The useful statement for a practitioner is therefore conditional: how much of a model an explanation can convey depends less on which explanation method is chosen than on how far the model's inputs depart from independence.

Applying the same logic to the evaluation metrics themselves turns out to be the more uncomfortable exercise. A metric that claims to say which of two explanations is better can be handed two explanations whose ordering is known by construction, and asked. Of the five established metrics we tested this way, two track the ordering reliably, one is defeated by the simplest possible worthless explanation (a permuted curve, which it prefers to the intact one) and two are reversed, systematically scoring degraded explanations above faithful ones. Our own metric lands in the first group, not statistically separable from the best performer ($p = \numDiscVsBestP$). 
That is a modest claim for the metric and a considerably less modest one about the state of explanation evaluation: face validity is not evidence that a metric orders explanations correctly, and the ordering can be tested directly by anyone willing to build a ladder of known quality and climb it.

\section*{Acknowledgements}
This work began as the first author's bachelor thesis at Oslo Metropolitan University.

\section*{Author contributions}
H.L.H. conceived the original idea, carried out the experiments, supervised the project, and wrote and reviewed the manuscript. J.S. carried out the experiments and wrote and reviewed the manuscript. Both authors read and approved the final manuscript.

\section*{Competing interests}
The authors declare no competing interests.

\section*{Funding}
The authors received no specific funding for this work.

\section*{Data availability}
The real datasets analysed in this study are publicly available from OpenML~\cite{vanschoren2014openml}; the dataset identifiers are listed in \cref{tab:datasets}. The synthetic designs are generated by the released code from the specifications in \cref{sec:experiments}.

\section*{Code availability}
All code needed to reproduce the experiments, tables and figures is publicly available at \url{https://github.com/hugohammer/xaieval}.

\bibliographystyle{unsrt}
\bibliography{references}

\end{document}